\documentclass[letterpaper, 10 pt, conference]{ieeeconf}  
\IEEEoverridecommandlockouts

\usepackage{graphics} 
\usepackage{epsfig} 
\usepackage{tabularx}
\usepackage{times} 
\usepackage{amsmath,amsthm} 
\usepackage{amssymb}  
\usepackage{booktabs}
\usepackage{multirow}
\usepackage[colorinlistoftodos]{todonotes}
\usepackage{float}
\usepackage{hyperref}
\usepackage{subcaption}
\usepackage{cite}

\usepackage[ruled,vlined,linesnumbered]{algorithm2e}
\usepackage{algpseudocode}
\usepackage{tikz}
\usepackage[dvipsnames]{xcolor}
\usetikzlibrary{calc}
\usetikzlibrary{arrows.meta, positioning}

\tikzset{
    ->, 
    >=Stealth,
    node distance=2cm and 2.5cm,
    every node/.style={draw, ellipse, minimum width=1.2cm, minimum height=0.8cm},
    every edge/.append style={draw, thick},
    label_node/.style={draw=none, fill=none, text=black}
}

\newcommand{\se}[1]{\textcolor{red}{#1}}

\newcommand{\tpo}{\phi}
\newcommand{\ctpo}{\phi_c}
\newcommand{\events}{\Pi}
\newcommand{\reset}{R}
\newcommand{\guardmap}{G}
\newcommand{\clocks}{C}

\newcommand{\apformulas}{\mathcal{F}}
\newcommand{\dts}{T}
\newcommand{\gtsptwpr}{\mathsf{G}}
\newcommand{\stpo}{\phi_s}
\newcommand{\Levent}{L_{\Pi}}
\newcommand{\Lprop}{L_\text{AP}}

\title{\LARGE \bf
Conditional Timed Partial Orders: An Expressive and Interpretable Framework for Robot Task Specification and Planning
}

\author{
Sebastian Escobar and Morteza Lahijanian
\thanks{Authors are with the Aerospace Eng. Sci. Dept. at University of Colorado Boulder, CO, USA
        {\tt\small \{firstname.lastname\}@colorado.edu}}%
}

\newtheorem{definition}{Definition}

\newtheorem{mytheorem}{Theorem}
\newtheorem{myproblem}{Problem}
\newtheorem{myexample}{Example}
\newtheorem{myproposition}{Proposition}

\begin{document}

\maketitle
\thispagestyle{plain}
\pagestyle{plain}

\begin{abstract}

Timed Partial Orders (TPOs), originally proposed for workflows, provide an interpretable framework for robot task specification with planning algorithms based on mixed-integer linear programming (MILP). However, TPOs are limited in expressivity, capturing only partial-order events with simple timing constraints.
In this paper, we introduce Conditional TPOs (cTPOs), which extend TPOs with richer relative-timing constraints and conditional event activations based on environmental conditions. We show that planning for cTPOs also reduces to an MILP problem; however, the added expressivity results in significantly larger MILPs that can become computationally intractable.
To address this challenge, we propose a decomposition algorithm that partitions a cTPO into smaller sub-TPOs, yielding a sequence of smaller MILP problems. We prove that this decomposition is complete and preserves plan optimality while improving the interpretability of complex tasks. Experimental results demonstrate the effectiveness of cTPOs as a task specification framework and the efficiency of our decomposition approach, achieving up to four orders of magnitude speedup over the monolithic MILP.
\end{abstract}
\section{Introduction}

Robotic autonomy increasingly requires reasoning at the mission level: robots must decide not only \emph{how} to move, but \emph{what} to do, \emph{when} to do it, and under which conditions tasks become required. Missions in domains such as healthcare, aerospace, and field robotics are structured by precedence, deadlines, and conditional obligations. 
However, specification languages capable of representing such requirements must balance competing goals of \emph{expressivity}, human \emph{interpretability}, \emph{formal rigor}, and \emph{computational tractability}~\cite{kress2018synthesis,guo2023recent}. In this work, we seek to improve this balance by extending an existing interpretable specification language while retaining formal guarantees and efficient planning.


In recent years, there has been growing interest in using natural language to specify robotic tasks and missions~\cite{singh2023progprompt,kim2024survey}. Advances in LLMs have further enabled the generation of plans directly from such specifications. These approaches offer substantial expressivity and human interpretability. However, LLM-based methods generally lack explicit formal semantics and, when used directly for planning, provide no guarantees of correctness and completeness~\cite{pmlr-v235-kambhampati24a}.



A widely used class of formal specification languages in robotics is temporal logic (TL)~\cite{kress2018synthesis}, including timed variants Metric/Signal TL (M/S-TL), which explicitly capture timing requirements~\cite{koymans1990specifying,maler2004monitoring}. These specifications admit rigorous planning methods based on timed-automata-theoretic~\cite{alur1994theory,maler2006mitl} or optimization-based reasoning~\cite{kress2018synthesis}, namely Mixed Integer Linear Program (MILP)~\cite{Cardona2023}. While expressive and formal, these specifications require domain expertise and familiarity with logical formulas, limiting their interpretability for untrained users. Also, the expressiveness of M/S-TL comes at a significant computational cost, making planning challenging.


A recent specification language, originally introduced for workflows, is \textit{Timed Partial Orders} (TPOs)~\cite{watanabe2023timed}, which provide a better balance among the aforementioned four competing objectives. TPOs represent missions as directed acyclic graphs (DAGs), where edges encode precedence relations and clocks impose timing constraints between events
(e.g., see Fig.~\ref{fig:hospital_tpo}). This graphical representation allows designers to interpret and modify task requirements directly, while retaining precise formal semantics. Moreover, planning under a TPO specification can be formulated as a variation of Generalized Traveling Salesman Problem 
\cite{noon1988generalized,mingozzi1997dynamic,pop2023comprehensive} and solved using MILP~\cite{watanabe2024optimalplanningtimedpartial}, avoiding expensive construction of timed automata~\cite{alur1994theory}. 
Nevertheless, existing TPOs have two key expressiveness limitations: (i) timing constraints are restricted to precedence-ordered events, and (ii) lack of support for conditional subtasks that depend on environmental events.

In this work, we address both gaps with \emph{conditional TPOs} (cTPOs). cTPOs extend TPOs with generalized clock-difference guards relating arbitrary pairs of events and propositional activation conditions that conditionally enforce tasks and timing constraints based on the realized trajectory, increasing expressivity. Then, we show that the MILP formulation of~\cite{watanabe2024optimalplanningtimedpartial} can be extended to plan for cTPO specifications. To mitigate the increased computational cost of this added expressivity, we introduce \emph{sub-TPO decomposition}, which partitions a cTPO into smaller sub-TPOs,
yielding a sequence of smaller MILP problems.
We prove that our decomposition algorithm is complete and preserves plan optimality, while substantially reducing computation time and improving interpretability by exposing the modular structure of the mission specification. Our case studies demonstrate the expressivity and interpretability of cTPOs and  benchmarks show that decomposition-based planning achieves speedups of up to four orders of magnitude.


In short, our contributions are fourfold: (i) we introduce cTPOs, extending TPOs with timing constraints between arbitrary events and conditional tasks; (ii) we formulate cTPO planning as an MILP; (iii) we develop a complete, optimality-preserving decomposition into sequentially solved sub-TPOs; and (iv) we demonstrate cTPO expressivity and interpretability on robotic missions and show up to four orders of magnitude speedup over the monolithic MILP.
\providecommand{\se}[1]{\textcolor{red}{#1}}

\section{TPOs \& Planning Problem}


We consider complex robotic tasks defined over a set of events with temporal requirements. A formalism that provides easily interpretable visual representations of such tasks is TPOs~\cite{watanabe2023timed}. 
Here, we first formally define TPOs and discuss their limitations in expressing complex tasks. We then introduce Conditional TPOs (cTPOs) to address these limitations and finally formulate the cTPO planning problem.



\subsection{Timed Partial Order (TPO)}
TPOs use partial-order relations together with clocks to capture temporal requirements over events.

\begin{definition}[TPO~\cite{watanabe2023timed}] \label{def:tpo}
    A \emph{Timed Partial Order} (TPO) is a tuple $\tpo = (\events, \prec, \clocks, \mathcal{G}_\clocks, \guardmap, \reset)$, where
    \begin{itemize}
    \item $\Pi = \{e_1, \ldots, e_n\}$ is a finite set of events; 
    \item $\prec \subseteq \events \times \events$ is a strict partial order over events in $\events$;
    \item $\clocks = \{c_1, \ldots, c_m\}$ is a finite set of clocks;
    \item $\mathcal{G}_\clocks = \big\{ \bigwedge_{i \in I} (c_i \bowtie t_i) \mid I \subseteq \{1,\ldots,|\clocks|\}, \, \bowtie \in \{<,\leq, \geq,>\}, \, t_i \in \mathbb{R}_{\geq 0} \big\}$ 
    is the set of timing constraints (guards), expressed as conjunctions of clock bounds;
    \item $\guardmap: \events \to \mathcal{G}_\clocks$ assigns to an event $e \in \Pi$ a timing constraint (guard condition) $\guardmap(e) \in \mathcal{G}_\clocks$; and 
    \item $\reset: \events \to 2^{\clocks}$ is the reset map that, upon execution of an event $e$, resets clocks $\reset(e) \subseteq \clocks$.
    \end{itemize}
\end{definition}

A TPO can be represented as a Directed Acyclic Graph (DAG) whose edges encode the cover relation of $\prec$.
%

\begin{myexample}[Hospital Robot Workflow]\label{ex:hospital}

    Consider a hospital robot tasked with collecting diagnostic results and delivering them to medical staff. The task is: first collect the triage report ($e_1$); then perform the blood draw ($e_2$) and imaging scan ($e_3$) in any order; finally deliver the results to a physician ($e_4$) and the patient file to a social worker ($e_5$) in any order. The blood draw $e_2$ and physician delivery $e_4$ must occur within 15 and 60 minutes, respectively, of $e_1$.

    The precedence relations are $e_1 \prec e_2$, $e_1 \prec e_3$, $e_2 \prec e_4$, $e_3 \prec e_4$, $e_2 \prec e_5$, and $e_3 \prec e_5$. Timing constraints use one clock $C=\{c_1\}$, reset at $e_1$ ($R(e_1) = \{c_1\}$), with guards $G(e_2)=(c_1 \leq 15)$ and $G(e_4)=(c_1 \leq 60)$. This TPO is shown as the black DAG with black clock and guard annotations in Fig.~\ref{fig:hospital_tpo}.
\end{myexample}

A \emph{timed trace} $\tau = (\sigma_1, t^{(1)}) \ldots (\sigma_n, t^{(n)})$ is a sequence of event-timestamp pairs, where each $\sigma_i \in \events$ is distinct and $t^{(i)} \in \mathbb{R}_{\geq 0}$.
A valid execution of a TPO is a timed trace that respects $\prec$ and timing constraints.  
To reason about the latter, we need to track clock readings, which we do via a \emph{valuation} function $\nu: \clocks \to \mathbb{R}_{\geq 0}$.  This is formalized below.

\begin{definition}[TPO Semantics~\cite{watanabe2024optimalplanningtimedpartial}]\label{def:tpo_semantics}
    A \emph{run} of TPO $\tpo$ is a sequence $\rho = (\sigma_1, t^{(1)}, \nu_1) \ldots (\sigma_n, t^{(n)}, \nu_n)$, where each $\sigma_i \in \events$ is distinct and each $\nu_j$ is a clock valuation, satisfying: 
    \begin{itemize}
        \item timestamps are non-decreasing, i.e., $t^{(1)} \leq \cdots \leq t^{(n)}$; 
        \item $\sigma_1, \ldots, \sigma_n$ is a linearization of $\prec$; 
        \item just before $\sigma_j$ fires (is executed), time duration $\Delta t^{(j)} = t^{(j)} - t^{(j-1)}$ elapses, advancing each clock $c \in C$ by $\Delta t^{(j)}$. Then, the resulting valuation 
        $(\nu \oplus \Delta t)(c) = \nu(c) + \Delta t$ must satisfy the guard $\guardmap(\sigma_j)$; and 
        \item after $\sigma_j$ fires, the clocks in $\reset(\sigma_j)$ are reset to zero while all others are unchanged, giving $\nu_j = \text{reset}(\nu_{j-1} \oplus \Delta t^{(j)}, \reset(\sigma_j))$, where 
        $\text{reset}(\nu, \hat{C})(c) = 0$ if $c \in \hat{C}$, otherwise $\nu(c)$.
    \end{itemize}
    A timed trace $\tau = (\sigma_1, t^{(1)}) \ldots (\sigma_n, t^{(n)})$ is \emph{compatible} with $\tpo$, denoted $\tau \models \tpo$, if there exists a corresponding run.
\end{definition}

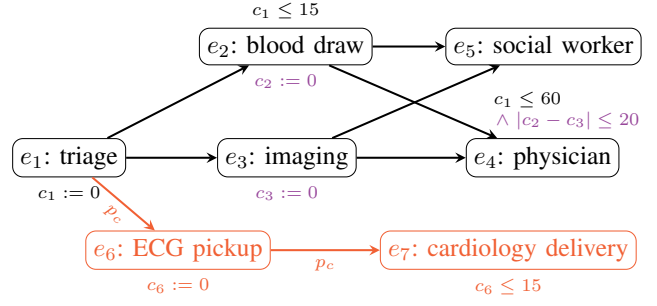
\begin{figure}[t]
    \centering
    \resizebox{\columnwidth}{!}{
    \begin{tikzpicture}[
        node distance=0.5cm and 3.5cm,
        every node/.style={draw, rounded corners, fill=white!20, minimum width=2em, inner sep=3pt},
        >=stealth
    ]
    \node (e1) {$e_1$: triage};
    \node (e2) [above right=1.0 and 1.0 of e1] {$e_2$: blood draw};
    \node (e3) [right=1.25 of e1] {$e_3$: imaging};
    \node (e4) [right=1.5 of e3] {$e_4$: physician};
    \node (e5) [right=1 of e2] {$e_5$: social worker};

    \node (e6) [below=1.0 of $(e1)!0.5!(e3)$, draw=RedOrange, text=RedOrange] 
        {$e_6$: ECG pickup};
    \node (e7) [right=1.5cm of e6, draw=RedOrange, text=RedOrange] 
        {$e_7$: cardiology delivery};

    \draw[->, thick] (e1) -- (e2);
    \draw[->, thick] (e1) -- (e3);
    \draw[->, thick] (e2) -- (e4);
    \draw[->, thick] (e3) -- (e4);
    \draw[->, thick] (e2) -- (e5);
    \draw[->, thick] (e3) -- (e5);

    \draw[->, thick, RedOrange] (e1) -- (e6) 
        node[midway, below, sloped, label_node, font=\scriptsize] 
        {\textcolor{RedOrange}{$p_c$}};
    \draw[->, thick, RedOrange] (e6) -- (e7) 
        node[midway, below, sloped, label_node, font=\scriptsize] 
        {\textcolor{RedOrange}{$p_c$}};

    \node[label_node, below=0.01cm of e1, font=\scriptsize] {$c_1 := 0$};
    \node[label_node, below=0.01cm of e2, font=\scriptsize] {\textcolor{Purple}{$c_2 := 0$}};
    \node[label_node, below=0.01cm of e3, font=\scriptsize] {\textcolor{Purple}{$c_3 := 0$}};
    \node[label_node, below=0.01cm of e6, font=\scriptsize] 
        {\textcolor{RedOrange}{$c_6 := 0$}};

    \node[label_node, above=0.01cm of e2, font=\scriptsize] {$c_1 \leq 15$};
    \node[label_node, above=0.01cm of e4, font=\scriptsize, align=center] {
        \hspace{-5mm} $c_1 \leq 60$  \\
        \hspace{6mm} \textcolor{Purple}{$\wedge$} \textcolor{Purple}{$|c_2 - c_3| \leq 20$}
    };
    \node[label_node, below=0.01cm of e7, font=\scriptsize] 
        {\textcolor{RedOrange}{$c_6 \leq 15$}};

    \end{tikzpicture}
    }
    \vspace{-6mm}
    \caption{
    \small
    TPO and cTPO specifications for the Hospital Robot examples. The TPO shown 
    in black
    captures the (original) requirements in Example~\ref{ex:hospital}. 
    It is extended to a cTPO by adding the purple clocks and guards to represent the additional requirement R1, and the bottom orange branch to capture requirement R2.}
    \label{fig:hospital_tpo}
    \vspace{-2mm}
\end{figure}

An example of a compatible timed trace with the TPO in Fig.~\ref{fig:hospital_tpo} is 
$\tau_0 = (e_1, 0)(e_2, 10)(e_3, 20)(e_4, 30)(e_5, 40)$.

While expressive and interpretable, TPOs in Def.~\ref{def:tpo} have two major limitations: restricted guard conditions and lack of support for environmental conditions.

\subsection{Conditional TPOs with Complex Guards}
\label{sec:ctpo}

TPOs in Def.~\ref{def:tpo} only allow guard conditions with single-clock bounds (see $G_\clocks$ in Def.~\ref{def:tpo}) and cannot express events triggered by external environmental conditions. For instance, in Example~\ref{ex:hospital}, additional requirements such as
\begin{itemize}
    \item[\bf{R1:}] ``\textit{Blood draw $e_2$ and imaging $e_3$ must be done within 20 minutes of each other.}'' 
    \item[\bf{R2:}] ``\textit{If the robot happens to pass near the Cardiac Wing, it should also pick up ECG report ($e_6$) and deliver it to cardiologist ($e_7$) within 15 minutes.}''
\end{itemize}
cannot be expressed.
We extend TPOs to support such expressions, enriching the properties they can represent while maintaining their strong interpretability.  These extensions, however, have consequences in computation, on which we elaborate in Sec.~\ref{sec:milp} and address in Sec.~\ref{sec:decomp}.

First, we generalize the guard condition $\mathcal{G}_\clocks$ to allow richer, \emph{precedence-free} constraints. Specifically, each guard may also be a conjunction of sums of clock differences, i.e.,

\vspace{-6mm}
\begin{multline}
    \label{eq: new guard}
    \mathcal{G}_\clocks'= 
    \mathcal{G}_\clocks \cup
    \Big\{\bigwedge_{k=1}^N \sum_{i_k \neq j_k \in I_k} (c_{i_k} - c_{j_k}) \bowtie t_k \mid 
    N \leq 2^{|C|}, \\  I_k \subseteq \{1,\ldots, |\clocks|\}, \, t_k \in \mathbb{R}_{\geq 0} \Big\}.
    \vspace{-1mm}
\end{multline}
This enables timing constraints between arbitrary pairs of events, regardless of their precedence relation.  Note that $\mathcal{G}_\clocks'$ allows for conditions such as $|c_2 - c_3| \leq 20$, which precisely captures the additional requirement R1 for Example~\ref{ex:hospital} by extending clocks to $\clocks = \{c_1,c_2,c_3\}$, as shown in Fig.~\ref{fig:hospital_tpo} (purple clocks and guards).

The second extension is motivated by the fact that 
real robotic tasks rarely unfold as a single fixed task structure. Environmental features encountered along the way (e.g., a geologically significant outcrop, or an area triggering a safety protocol) may impose additional tasks with their own precedence and timing requirements, e.g., R2 for Example~\ref{ex:hospital}. 



To capture such behavior, we associate events with logical conditions over the environment. Let $AP = \{p_1, \ldots, p_r\}$ be a finite set of atomic propositions (predicates), where each $p_i$ evaluates to true when the robot's trajectory satisfies an associated environmental condition (e.g., entering a designated area). Furthermore, let $\apformulas(AP)$ denote the set of all Boolean (propositional) formulas over $AP$ formed using $\neg$ (negation), $\wedge$ (conjunction), and $\vee$ (disjunction). 
We allow each event in $\Pi$ to be triggered by a formula in $\apformulas(AP)$.
Intuitively, an event and its timing and precedence constraints are triggered in the task structure of TPO only when its associated formula evaluates to true. Otherwise, the event and its precedence and temporal constraints are ignored. We formalize this extension as a \emph{Conditional TPO} (cTPO).

\begin{definition}[cTPOs]\label{def:ctpo}
    A \emph{Conditional TPO} (cTPO) is a tuple 
    $\ctpo = (\tpo, AP, \alpha)$
    where 
    $\tpo$ is a TPO in Def.~\ref{def:tpo} with guard condition $\mathcal{G}_\clocks'$ in \eqref{eq: new guard}, $AP$ is a finite set of atomic propositions, and $\alpha : \events \to \apformulas(AP)$ is an activation function that assigns, to each event $e \in \events$, a propositional condition (Boolean formula) $\alpha(e) \in \apformulas(AP)$ over $AP$. 
\end{definition}

When $\alpha(e)$ evaluates to true, i.e., $\alpha(e) \equiv \top$, event $e$ and its associated precedence constraints, guard conditions, and clock resets are enforced; otherwise, $e$ is not required. Note that cTPO $\ctpo$ recovers TPO $\tpo$ if $\alpha(e) = \top$ for all $e \in \events$.

The TPO in Fig.~\ref{fig:hospital_tpo} with the bottom orange branch is a cTPO for Example~\ref{ex:hospital} with R1 and R2 requirements, where $\events = \{e_1,\ldots,e_7\}$, $AP = \{p_c\}$, and $\alpha(e_i) = p_c$ if $i \in \{6,7\}$ and $\alpha(e_i) = \top$ if $i\in \{1,\ldots,5\}$ (not shown in the figure).  



For each evaluation of the propositions in $AP$, a cTPO induces a TPO obtained by restricting to the events whose activation conditions evaluate to true. Hence, cTPO semantics are inherited from Def.~\ref{def:tpo_semantics} on this induced substructure.



\begin{definition}[cTPO Semantics]\label{def:ctpo_semantics}

    Let $\omega : AP \to \{\bot, \top\}$ be a valuation of atomic propositions in $AP$.
    The \emph{active event set} induced by $\omega$ is $\events_\omega := \{e \in \events \mid \alpha(e) \equiv \top \text{ under } \omega\}$, and  \emph{active suborder} induced by $\omega$ is $\prec_\omega :=$ $\prec \cap (\events_\omega \! \times \events_\omega)$.
    A \emph{run} of a cTPO $\ctpo$ under $\omega$ is a sequence $\rho = (\sigma_1, t^{(1)}, \nu_1) \ldots  (\sigma_k, t^{(k)}, \nu_k)$, where $k = |\events_\omega|$ and $\sigma_i \in \events_\omega$, satisfying the conditions in Def.~\ref{def:tpo_semantics} with $\prec$ replaced by $\prec_\omega$. A timed trace $\tau$ is \emph{compatible} with $\ctpo$ under $\omega$, denoted $\tau \models \ctpo$, if there exists a corresponding run.
\end{definition}

An example of a compatible timed trace that satisfies the cTPO in Fig.~\ref{fig:hospital_tpo} (Example~\ref{ex:hospital}) under valuation $\omega(p_c) = \top$ is 
$\tau_1 = (e_1, 0)(e_2, 10)(e_3, 20)(e_4, 30)(e_5, 40)(e_6, 50)(e_7, 60)$.

\subsection{Robot Model}
Similar to~\cite{watanabe2024optimalplanningtimedpartial}, we consider a robot whose evolution in the environment is modeled as a Deterministic Transition System (DTS), which is a common abstraction in the literature~\cite{kress-gazit2007wheres-waldo, kress2018synthesis, lahijanian2009}.


\begin{definition}[DTS]\label{def:dts}
    A deterministic transition system (DTS) is a tuple $\dts = (X, x_0, A, \delta, \Delta, \events, \Levent, AP, \Lprop)$, where 
    \begin{itemize}
        \item $X$ is a finite set of states, 
        \item $x_0 \in X$ is the initial state,
        \item $A$ is a finite set of actions, 
        \item $\delta: X \times A \to X$ is the (partial) transition function, 
        \item $\Delta: X \times A \to \mathbb{R}_{\geq 0}$ is the transition duration function, 
        \item $\events$ is a finite set of events, 
        \item $\Levent: X \to \events \cup \{\emptyset\}$ is an event labeling function that maps each state to an event with $L(x_0) = \emptyset$,
        \item $AP = \{p_1, \ldots, p_r\}$ is a finite set of atomic propositions, and
        \item $\Lprop: X \to 2^{AP}$ is a proposition labeling function that maps each state to the atomic propositions it satisfies.
    \end{itemize}
\end{definition}

A \textit{plan} $\gamma = \gamma_0 \gamma_1 \ldots \gamma_{n-1}$ with $\gamma_i \in A$ is a sequence of actions. A plan $\gamma$ is \textit{valid} if it generates a trajectory $s^\gamma = s^\gamma_0 s^\gamma_1 \ldots s^\gamma_n$ such that $s^\gamma_0 = x_0$ and $s^\gamma_{i+1} = \delta(s^\gamma_i, \gamma_i)$.  
The set of all valid plans is denoted $\Gamma$.
The \textit{duration} of $s^\gamma$ is defined as the sum of the transition durations along the trajectory, i.e., $D(s^\gamma) = \sum_{i=0}^{n-1} \Delta(s^\gamma_i, \gamma_i)$. 

A valid plan $\gamma$ induces a timed trace $\tau^\gamma = (L(s^\gamma_0), t_0) \ldots (L(s^\gamma_n), t_n)$, where $t_i = D(s^\gamma_0 \ldots s^\gamma_i)$. 
We say plan $\gamma$ \emph{satisfies} a TPO $\tpo$ defined over events $\events$ if $\tau^\gamma \models \tpo$ per Def.~\ref{def:tpo_semantics}. 
For a cTPO $\ctpo = (\tpo, AP, \alpha)$, we first define valuation $\omega$ for every $p_k \in AP$ to be 
\begin{align}
    \label{eq: valuation function}
    \omega(p_k):= 
        \begin{cases}
            \top & \text{if } p_k \in \Lprop(s^\gamma_i) \text{ for some }  i\\
            \bot & \text{otherwise.}
        \end{cases}
\end{align}
Then, we say $\gamma$ satisfies $\ctpo$ if $\tau^\gamma$ is compatible with $\ctpo$ under $\omega$ per Def.~\ref{def:ctpo_semantics}.

Our goal is to compute a minimum-duration plan that satisfies a given cTPO specification, as formalized below.

\begin{myproblem}[cTPO Plan Synthesis]\label{prob:synthesis}
Given a robot as a DTS $\dts$ 
in Def.~\ref{def:dts}
and
a cTPO task $\ctpo$ 
as in Def.~\ref{def:ctpo}, 
synthesize a valid plan $\gamma^* \in \Gamma$
that 
generates a minimum-duration trajectory satisfying $\ctpo$, i.e., $\gamma^* = \arg\min_{\gamma \in \Gamma} D(s^\gamma) \quad \text{subject to} \quad \tau^\gamma \models \ctpo.$
\end{myproblem}
\providecommand{\se}[1]{\textcolor{red}{#1}}

\section{Approach}\label{sec:approach}


Our approach to Problem~\ref{prob:synthesis} proceeds in three stages. First, we express cTPO compatibility in terms of event timestamps, reducing the problem to a Generalized Traveling Salesman Problem with Time Windows and Precedence Relations (GTSP-TWPR)~\cite{watanabe2023timed,watanabe2024optimalplanningtimedpartial,Lahijanian:RAL-tempral:2018} (Sec.~\ref{sec:gtsp}). Second, we solve the resulting problem exactly via an MILP that extends~\cite{watanabe2024optimalplanningtimedpartial} with proposition tracking and conditional activation constraints (Sec.~\ref{sec:milp}). Third, we improve scalability through a sub-TPO decomposition that solves self-contained cTPO substructures as smaller MILPs and composes their solutions into a globally optimal plan (Sec.~\ref{sec:decomp}).

\subsection{Reduction to GTSP-TWPR}\label{sec:gtsp}

Whether a plan satisfies a cTPO depends only on which events are active, which event-labeled states are visited, in what order, and at what times. Problem~\ref{prob:synthesis} thus reduces to a routing problem: select one location per required event, order the visits, and schedule each to satisfy the timing constraints. To formulate these scheduling bounds, the timing constraints in $\guardmap$ are expressed directly over visit timestamps. Since clocks reset only at events, the value of any clock when an event fires is simply the elapsed time since its most recent reset, i.e., the difference of two event timestamps. Substituting this into every guard conjunct rewrites each timing requirement as a linear inequality over visit timestamps, yielding the following characterization.

\begin{myproposition}[Timestamp Characterization]
\label{prop:timestamps}
Let $\ctpo$ be a cTPO and $\omega : \mathit{AP} \to \{\bot, \top\}$ a fixed valuation. The guard conditions of a run of $\ctpo$ under $\omega$ are satisfied iff there exist constants $a_i, b_i, \ell_{ij}, u_{ij}, f_{kl}, w_{kl} \in \mathbb{R}_{\geq 0} \cup \{\infty\}$ determined by $\guardmap$ and $\reset$ such that
\begin{subequations}
    \begin{align}
      &t_i \in [a_i,\, b_i] && \forall e_i \in \events_\omega, \label{eq:time_abs} \\
      &t_j - t_i \in [\ell_{ij},\, u_{ij}] && \forall e_i,e_j \in \events_\omega \text{ s.t. } e_i\prec_\omega e_j, \label{eq:time_prec} \\
      &t_l - t_k \in [f_{kl},\, w_{kl}] && \forall (e_k, e_l) \in \events_\omega \!\!\times \! \events_\omega \text{ coupled\;by\;a} \nonumber \\ 
      & && \text{ clock-difference\;of\;} \mathcal{G}_\clocks'\,\text{in}~\eqref{eq: new guard}  \label{eq:time_prec_free} 
    \end{align}
\end{subequations}
where $t_i \in \mathbb{R}_{\geq 0}$ is the timestamp at which event $e_i$ fires.
\end{myproposition}

Condition~\eqref{eq:time_abs} is an \emph{absolute time window} per event, derived from guards whose clocks reset at mission start; \eqref{eq:time_prec} is a \emph{relative time window} between precedence-ordered pairs, derived from guards whose clocks reset at the earlier event; and \eqref{eq:time_prec_free} is a \emph{precedence-free relative time window} between clock-coupled unordered pairs, the novel case admitted by $\mathcal{G}_\clocks'$ in~\eqref{eq: new guard}. These three families correspond one-to-one to the scheduling constraints of a GTSP-TWPR.
With this reduction, the objective becomes finding a tour of the GTSP-TWPR graph that visits one vertex per required event and satisfies the ordering given by $\prec$ as well as the conditions in~\eqref{eq:time_abs}-\eqref{eq:time_prec_free}. 
Below, we first formulate this GTSP-TWPR by extending the one in~\cite{watanabe2024optimalplanningtimedpartial},
and then show its construction from the robot DTS and cTPO $\ctpo$. 

\begin{definition}[GTSP-TWPR]
\label{def:gtsp}
    Let $\gtsptwpr = (V, E)$ be a weighted directed graph with depot $v_0 \in V$, and denote 
    by $d_i \in \mathbb{R}_{\geq 0}$ the dwell cost of (time spent at) $v_i \in V$ and
    by $d_{ij} \in \mathbb{R}_{\geq 0}$ the travel cost (time) of edge $(v_i, v_j) \in E$. 
    The nodes in $V \setminus \{v_0\}$ are partitioned into disjoint sets (groups) $V_1, \ldots, V_m$, each designated \emph{active} or \emph{inactive}.
    GTSP-TWPR imposes a partial order over the groups, a time window $[a_i,b_i]$ constraint over each $V_i$, and relative time window $[\ell_{ij}, u_{ij}]$ over coupled pairs $(V_i,V_j)$.
    Let $t_i$
    denote the visit time of a node in the active set $V_i$. 
    Then,
    a satisfying \emph{tour} starts and ends at $v_0$, visits exactly one node from each active group, and minimizes the makespan subject to:
    \begin{itemize}
      \item \emph{time windows:} $t_i \in [a_i, b_i]$ for each active group $V_i$;
      \item \emph{precedence constraints}: $t_i \leq t_j$ for each ordered pair of active groups $(V_i, V_j)$; and
      \item \emph{relative time windows}: $t_j - t_i \in [\ell_{ij}, u_{ij}]$ for each coupled pair of active groups $(V_i, V_j)$.
    \end{itemize}
\end{definition}

An example of $\gtsptwpr$ for the robot in Example~\ref{ex:hospital} is shown in Fig.~\ref{fig: gtsp eaxmple}.
We now show how to construct a GTSP-TWPR graph $\gtsptwpr$ for a given $\ctpo$ using DTS $\dts$. 
Nodes in $V$ and edges in $E$ correspond to states and transitions of $\dts$ that are relevant to $\ctpo$, and
each node groups $V_i$ is associated with an event in $\events$ or proposition in $AP$, as detailed below.


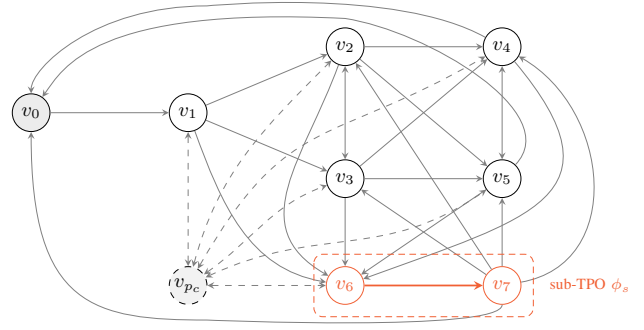
\begin{figure}[t]
      \centering
      \resizebox{0.95\columnwidth}{!}{%
      \begin{tikzpicture}[                                     
          every node/.style={draw, circle, minimum size=6mm, inner sep=0.5pt, font=\small},
          >=stealth,
          inc/.style={<->, thin, gray},                        
          prop/.style={<->, thin, gray, dashed},
          x=1cm, y=1cm
      ]
                
      \node (v0) at (0,0)       [fill=gray!15] {$v_0$};
      \node (v1) at (2.5,0)     {$v_1$};
      \node (v2) at (5,1.05)  {$v_2$};
      \node (v3) at (5,-1.05) {$v_3$};                                                   
      \node (v4) at (7.5,1.05)  {$v_4$};
      \node (v5) at (7.5,-1.05) {$v_5$};

      \node (v6) at (5,-2.75) [draw=RedOrange, text=RedOrange] {$v_6$};
      \node (v7) at (7.5,-2.75) [draw=RedOrange, text=RedOrange] {$v_7$};

      \node (vp) at (2.5,-2.75) [dashed, fill=gray!10] {$v_{p_c}$};

      \draw[->, thin, gray] (v0) -- (v1);
      \draw[->, thin, gray] (v4) .. controls (5.8,2) and (2.8,1.7) .. (2.8,1.7)
                             .. controls (1.0,1.7) and (0,.7) .. (v0);
      \draw[->, thin, gray] (v5) .. controls (8.5,0) and (6.3,1.3) .. (5,1.5)
                             .. controls (1.0,1.7) and (0.5,.7) .. (v0);

      \draw[->, thin, gray] (v1) -- (v2);
      \draw[->, thin, gray] (v1) -- (v3);
      \draw[->, thin, gray] (v2) -- (v4);
      \draw[->, thin, gray] (v3) -- (v4);                                                       
      \draw[->, thin, gray] (v2) -- (v5);
      \draw[->, thin, gray] (v5) -- (v6);
      \draw[->, thin, gray] (v3) -- (v5);
      \draw[->, thin, gray] (v1) to[out=-70, in=170] (v6);
      \draw[->, thick, RedOrange] (v6) -- (v7);
      \draw[->, thin, gray] (v7) -- (v5);
      \draw[->, thin, gray] (v7)  .. controls (7.5,-3.5) and (2.8,-3.3) .. (2.8,-3.3)
                             .. controls (1.0,-3.3) and (0,-3.1) .. (v0);

      \draw[inc] (v2) -- (v3);
      \draw[inc] (v4) -- (v5);
      \draw[->, thin, gray] (v3) -- (v6);                                 
      \draw[->, thin, gray] (v2) to[out=-110, in=150, looseness=1.2] (v6);
      \draw[->, thin, gray] (v4) to[out=-50, in=20, looseness=1.7] (v6);
      \draw[->, thin, gray] (v7) -- (v2);
      \draw[->, thin, gray] (v7) -- (v3);
      \draw[->, thin, gray] (v7) to[out=10, in=-30, looseness=1.2] (v4);

      \draw[prop] (v1) -- (vp);
      \draw[prop] (vp) -- (v6);    
      \draw[prop] (vp) to[out=75,  in=-135] (v2);
      \draw[prop] (vp) to[out=35,  in=-160] (v3);
      \draw[prop] (vp) to[out=25, in=-145, looseness=1.1] (v5);
      \draw[prop] (vp) to[out=60,  in=-150, looseness=1.3] (v4);

      \draw[RedOrange, densely dashed, rounded corners]                                    
          ($(v6.north west)+(-0.28,0.28)$) rectangle ($(v7.south east)+(0.28,-0.28)$);
      \node[label_node, text=RedOrange, anchor=west, font=\scriptsize]
          at ($(v7.east)+(0.42,0)$) {sub-TPO $\stpo$};         
      \end{tikzpicture}}
      \vspace{-2mm}
      \caption{
      \small
      GTSP-TWPR $\gtsptwpr$ for the Hospital Robot in Example~\ref{ex:hospital} with cTPO in Fig.~\ref{fig:hospital_tpo}. 
      $v_0$ is the depot, 
      $v_1,\ldots,v_7$ are event nodes, and
      $v_{p_c}$ is the (cardiac-wing) 
      proposition node. 
      }
      \label{fig: gtsp eaxmple}
\end{figure}

\textbf{Nodes:}
We associate each $v \in V$ with $x \in X$ whose label is non-empty, and the depot $v_0 = x_0$, i.e.,
\begin{equation}
  V = \{x \in X \mid x = x_0 \text{ or } \Levent(x) \neq \emptyset \text{ or } \Lprop(x) \neq \emptyset\}.
  \label{eq:V}
\end{equation}
Event group $V_{l_i} = \{x \in X \mid \Levent (x) = e_i\}$ is a set of states with label $e_i$. Proposition group $V_{\ell_k} = \{x \in X \mid p_k \in \Lprop(x),\, \Levent(x) = \emptyset\}$ is the set of states whose proposition label includes $p_k$ and event label is empty. 
Note that $V_{l_1},\ldots,V_{l_{|\events|}},\cup_{k=1}^{|AP|}V_{\ell_k}$ constitute a partition of $V \setminus \{v_0\}$.

\textbf{Edges:}
Edges $E$ encode which transitions between nodes are permitted, and are determined jointly by the precedence structure of $\ctpo$ and the reachability in $\dts$. Let $\Pi_{\prec}^I,\Pi_{\prec}^F \subseteq \Pi$ denote the sets of events with \textit{no predecessors} and \textit{no successors} in $\prec$, respectively. Then, $E$ is defined as follows.

\vspace{-1mm}
\begin{table}[ht]
    \centering
    \begin{tabular}{l lcll}
        \hline
        & \textbf{Source} & \textbf{Edge} & \textbf{Target} & \textbf{Condition} \\
        \hline
        $1$\quad & $v_0$ & $\to$ & $v_j \in V_{l_i}$ & $e_i \in \Pi_{\prec}^I$ \\
        $2$ & $v_i \in V_{l_i}$ & $\to$ & $v_0$ & $e_i \in \Pi_{\prec}^F$ \\
        $3$\; & $v_i \in V_{l_i}$ & $\to$ & $v_j \in V_{l_j}$ & $e_i \prec e_j$ \\
        $4$\; & $v_i \in V_{l_i}$ & $\leftrightarrow$ & $v_j \in V_{l_j}$ & $e_i \not\prec e_j,\ e_j \not\prec e_i$ \\
        $5$\; & $v_i \in V_{\ell_k}$ & $\leftrightarrow$ & $v_j \in V_{l_j}$ & $\forall\, k, j$ \\
        $6$\; & $v_i \in V_{\ell_k}$ & $\leftrightarrow$ & $v_j \in V_{\ell_{k'}}$ & $\forall\, k \neq k'$ \\
        \hline
    \end{tabular}
    \vspace{-2mm}
\end{table}

Rows~1-2 restrict the depot to events in $\Pi_{\prec}^I$ and $\Pi_{\prec}^F$: the tour must begin at an event with no required predecessors and end at one with no required successors. Row~3 adds a directed edge for each precedence-ordered pair, carrying a relative time window from~\eqref{eq:time_prec} when applicable. Row~4 adds bidirectional edges between events unordered under $\prec$, carrying a time window from~\eqref{eq:time_prec_free} when applicable. Rows~5-6 connect proposition nodes to all event and proposition groups, enabling proposition detection along any path through the graph.  Note that, with this construction, directed edges enforce precedence order $\prec$ of $\ctpo$.

\textbf{Costs: }
Travel costs $d_{ij}$ between any two nodes $v_i, v_j \in V$ are the shortest-path durations computed via graph search (e.g., A* and Floyd-Warshall algorithm) over all states in $X$ of $\dts$. The dwell cost is $d_i = \Delta(x_i, a)$ if $v_i$ has a self-loop action $a \in A$, else $d_i = 0$. 


This construction of $\gtsptwpr$ encodes precedence order of $\ctpo$, but
the timing windows from Proposition~\ref{prop:timestamps} and proposition tracking must still be enforced. We do this through an MILP.

\subsection{MILP Formulation}\label{sec:milp}

With the GTSP-TWPR instance $\gtsptwpr$ in hand, we solve it exactly via an MILP, which simultaneously selects a minimum-makespan tour, infers which propositions hold along it, and enforces timing and coverage constraints only for the resulting active events. 
The complete formulation is given in Fig.~\ref{fig:milp}. 
%
%
%
Constraints~\eqref{eq:vars}--\eqref{eq:tpo_prec} and~\eqref{eq:travel}--\eqref{eq:makespan}, together with the routing variables $y_{ij}$ and visit-time variables $t_i$, extend the MILP of~\cite{watanabe2024optimalplanningtimedpartial}; the novel contributions of this work are the precedence-free timing constraints~\eqref{eq:tpo_pf}, the proposition and activation variables $z_k, \chi_e$, and the proposition-tracking and conditional-activation constraints~\eqref{eq:z_vars}--\eqref{eq:z_pf}, described below.

\textbf{Variables.} Let $n = |\events|$, $N = |V|$, and $r = |AP|$. Further, for any positive integer $k$, let $[k] = \{1,\ldots,k\}$. Then, $[N]$ indexes all non-depot nodes (including proposition nodes), with $\mathcal{P} \subseteq [N]$ the subset of proposition node indices.
The routing variables are $y_{ij} \in \{0,1\}$, indicating whether edge $(v_i,v_j) \in E$ is traversed, and 
$t_i \geq 0$, the visit time at node $v_i$, making $t_{N+1} \geq 0$ the makespan.

\textbf{Proposition and activation variables.} For each proposition $p_k \in AP$, the binary variable $z_k \in \{0,1\}$ is inferred from the tour: $z_k = 1$ iff the tour passes through some node in $V_{\ell_k}$. Let $Z = \{z_1, \ldots, z_r\}$, and $\Pi_c \subseteq \events$ denote the \emph{conditional events}. For each $e \in \Pi_c$, we introduce a single binary variable $\chi_e \in \{0,1\}$, obtained by linearizing $\alpha(e)$ over $Z$ such that $\chi_e = 1$ iff $\alpha(e) \equiv \top$ (see Table~\ref{tab:linearization}).

In Example~\ref{ex:hospital}, if the tour passes through the cardiac wing ($V_{\ell_{p_c}}$), then $z_{p_c} = 1$, which sets $\chi_{e_6} = \chi_{e_7} = 1$ and activates the ECG pickup $e_6$ and delivery $e_7$.

\textbf{Constraints.} Flow conservation~\eqref{eq:flow} enforces tour continuity. Each active event group must be visited exactly once~\eqref{eq:in}-\eqref{eq:out}, while proposition nodes carry no coverage constraint and may be traversed freely. The timing windows from Proposition~\ref{prop:timestamps} are enforced as linear inequalities on visit times: \eqref{eq:tpo_prec} encodes bounds from~\eqref{eq:time_prec} and \eqref{eq:tpo_pf} those from~\eqref{eq:time_prec_free}. Constraint~\eqref{eq:travel} ensures that if edge $(v_i, v_j)$ is traversed, the visit time at $v_j$ accounts for travel and dwell at $v_i$. The makespan is bounded by~\eqref{eq:makespan}, which may be omitted if the robot need not return to the depot, relaxing~\eqref{eq:flow} at $v_0$.

\begin{figure}[t]
    \centering
    \footnotesize
    \begin{subequations}
    \begin{align}
        & \min\ t_{N+1}
            && \nonumber \\
        & y_{ij} \in \{0,1\}  
            && \forall (i,j) \in E \label{eq:vars} \\
        & t_i \geq 0 
            && \forall i \in \{0\} \cup [N+1] \label{eq:tvars} \\
        & \sum_{i} y_{ij} - \sum_{k} y_{jk} = 0 
            &&  \forall j \in \{0\} \cup [N] \label{eq:flow} \\
        & \sum_{v_{l_j} \in V_l} \sum_{i} y_{ij} = 1 
            && \forall l \in [n],e_l \notin \Pi_c, j \not\in \mathcal{P}
            \label{eq:in} \\
        & \sum_{v_{l_j} \in V_l} \sum_{k} y_{jk} = 1 
            && \forall l \in [n],e_l \notin \Pi_c, j \not\in \mathcal{P} 
            \label{eq:out} \\
        & t_j - t_i \in [\ell_{ij},\, u_{ij}]
            && \forall (e_i,e_j)\in\eqref{eq:time_prec}, e_i, e_j \notin \Pi_c
            \label{eq:tpo_prec} \\
        & t_l - t_k \in [f_{kl},\, w_{kl}]
            && \forall (e_k,e_l)\in\eqref{eq:time_prec_free}, e_k, e_l \notin \Pi_c
            \label{eq:tpo_pf} \\
        & y_{ij}{=}1 \Rightarrow t_j - t_i \geq d_{ij} + d_j 
            && \forall i \in \{0\} \cup[N], \forall j\in[N] \setminus \mathcal{P}
            \label{eq:travel} \\
        & y_{ij}=1 \Rightarrow t_{N+1} - t_i \geq d_{ij} 
            && \forall i\in[N], j = 0
            \label{eq:makespan} \\
        & z_k,\ \chi_{e_l} \in \{0,1\} 
            && \forall k \in [r],\ \forall e_l \in \Pi_c \label{eq:z_vars} \\
        & z_k \geq y_{ij} 
            && \forall k \in [r],\ (i,j)\in E_k
            \label{eq:z_lb} \\
        & z_k \leq \textstyle\sum_{(i,j) \in E_k} y_{ij} 
            && 
            \forall k \in [r]
            \label{eq:z_ub} \\
        & \sum_{v_{l_j} \in V_l} \sum_{i} y_{ij} = \chi_{e_l}
            && \forall e_l \in \Pi_c
            \label{eq:z_cov} \\
        & \chi_{e_i}=1 \Rightarrow t_j - t_i \in [\ell_{ij},\, u_{ij}] \!
            && \forall (e_i,e_j)\in\eqref{eq:time_prec}, e_i \in \Pi_c
            \label{eq:z_tpo} \\
        & \chi_{e_k}=1 \Rightarrow t_l - t_k \in [f_{kl},\, w_{kl}] \!
            && \forall (e_k,e_l)\in\eqref{eq:time_prec_free}, e_k \in \Pi_c \label{eq:z_pf}
    \end{align}
\end{subequations}
    \caption{
    \small
    MILP formulation, where $e_l$ is an event of group $V_l$, with activation indicator $\chi_{e_l}$ when $e_l \in \Pi_c$, $V_{\ell_k}$ is the proposition group for $p_k$, and $E_k = \{(i,j) \in E \mid p_k \in \Lprop(x_i) \text{ or } p_k \in \Lprop(x_j)\}$.}
    \label{fig:milp}
\end{figure}

The novel constraints and variables couple routing and activation: 
the tour determines which propositions hold; this in turn determines the conditional events that are active and what timing they must satisfy, all resolved within a single solve. 
Specifically, \emph{proposition tracking}~\eqref{eq:z_lb}-\eqref{eq:z_ub} infers $z_k$ from the tour; \emph{conditional coverage}~\eqref{eq:z_cov} ties event group visitation to $\chi_{e_l}$; and \emph{conditional timing}~\eqref{eq:z_tpo}-\eqref{eq:z_pf} enforces timing windows only for active events. Table~\ref{tab:linearization} details how $\chi_e$ is computed from $Z$ following standard MILP linearization of Boolean formulas~\cite{Cardona2023} for each Boolean connective.

\begin{table}[t]
\caption{
\small
Linearization of formula $\alpha(e)$ into binary constraints relating the activation indicator $\chi_e$ to prop. variables $z_{k_i}, z_{k_j} \in Z$.}
\centering\small
\begin{tabular}{ll}
\hline
\textbf{Formula} $\alpha(e)$ & \textbf{Linearization of} $\chi_e$ \\
\hline
$p_{k_i}$ & $\chi_e = z_{k_i}$ \\
$\neg p_{k_i}$ & $ \chi_e = 1 - z_{k_i}$ \\
$p_{k_i} \wedge p_{k_j}$ & $\chi_e \leq z_{k_i},\ \chi_e \leq z_{k_j},\ \chi_e \geq z_{k_i} + z_{k_j} - 1$ \\
$p_{k_i} \vee p_{k_j}$ & $\chi_e \geq z_{k_i},\ \chi_e \geq z_{k_j},\ \chi_e \leq z_{k_i} + z_{k_j}$ \\
\hline
\end{tabular}
\label{tab:linearization}
\end{table}

Nested formulas are handled recursively by applying these rules to each subformula.

\subsubsection{Complexity}\label{sec:milp_complexity}
GTSP-TWPR is NP-hard~\cite{Garey+Johnson/1979/Computers}; solving it via MILP has exponential worst-case time complexity in the number of binary variables and, per branch-and-bound node, polynomial in the number of constraints. Our MILP has $\mathcal{O}(N^2 + n \cdot r)$ binary variables
and $\mathcal{O}(r \cdot N^2)$ constraints.

Timing constraints have a dual effect: each additional constraint tightens the LP relaxation and strengthens the branch-and-bound bound, but also increases solve time per node. Consequently, sparser timing constraints can increase solve time in practice, as confirmed in Sec.~\ref{sec:experiments}. As $N$ and $r$ grow with $\ctpo$ complexity, solving a single MILP over a larger event set can become intractable.  We address this issue next.

\subsection{Sub-TPO Decomposition}\label{sec:decomp}

Rather than solving a single large MILP over all events (the \emph{monolithic} formulation), we observe that many task structures contain subsets of events that are internally coupled, through both precedence and timing constraints, but interact with the rest of the plan only through a single entry and exit event.
Such substructures can be identified, solved independently as smaller MILP instances, and composed back into the global plan, which we call the \emph{decomposition} approach.
Since each sub-problem is exponential only in its own size, and the sub-problem sizes are small relative to the full problem, the overall computational cost becomes a sum of small exponential problems rather than one large one.
We formalize these substructures as sub-TPOs.

\begin{definition}[Sub-TPO]\label{def:stpo}
    Let $\ctpo$ be a cTPO with event set $\events$ and $\gtsptwpr = (V, E)$ the corresponding GTSP-TWPR graph. Let $\events_s \subseteq \events$ and let $\stpo$ be a cTPO with event set $\events_s$, whose precedence, timing, and activation components are defined over events in $\events_s$. We say $\stpo$ is a \emph{sub-TPO} of $\ctpo$ if: 
    \begin{enumerate}
        \item  it has a unique \emph{entry} $e_{\mathrm{in}}$ and \emph{exit} $e_{\mathrm{out}}$ in $\events_s$, i.e., $e_{\mathrm{in}}$ is the unique event in $\events_s$ that is preceded by some event in $\events \setminus \events_s$, and $e_{\mathrm{out}}$ is the unique event in $\events_s$ that precedes some event in $\events \setminus \events_s$, and
        
        \item each internal event $e_i \in \events_s \setminus \{e_{\mathrm{in}}, e_{\mathrm{out}}\}$ satisfies:
        \begin{enumerate}
            \item \emph{GTSP-TWPR isolation:} every node in $V_{l_i}$ is connected in $E$ only to nodes belonging to events in $\events_s$, with no edges to proposition nodes; further, $V_{l_{\mathrm{out}}}$ has no incoming edge
            from a node of an event in $\events \setminus \events_s$.
            \label{cond: G isolocation}
            \item \emph{Sub-TPO isolation:} $e_i$ is unordered under $\prec$ with respect to every $e_j \notin \events_s$.
            \label{cond: sub-top isolation}
            \item \emph{Clock isolation:} the clocks appearing in $\guardmap(e_i)$ and $\reset(e_i)$ are not used by $\guardmap(e_j)$ or $\reset(e_j)$ for any $e_j \notin \events_s$.
            \label{cond: clock isolation}
            \item \emph{Activation consistency:} $\alpha(e_i) \! \equiv \! \alpha(e_{\mathrm{in}}) \! \equiv \! \alpha(e_{\mathrm{out}})$.
            \label{cond: activation consistency}
        \end{enumerate}
    \end{enumerate}
    \vspace{-3mm}
\end{definition}


The cTPO in Fig.~\ref{fig:hospital_tpo} with $\gtsptwpr$ in Fig.~\ref{fig: gtsp eaxmple} contains one sub-TPO, which happens to be exactly the part shown in orange.

Among all valid sub-TPOs rooted at a given entry event, we seek a largest one to maximize the computational benefit of decomposition. The existence and structure of sub-TPOs is determined by the topology of $\gtsptwpr$, which depends on $\dts$.
\begin{definition}[Locally Maximal Sub-TPO]\label{def:max_stpo}
A sub-TPO $\stpo$ of $\ctpo$ with entry $e_{\mathrm{in}}$ is \emph{locally maximal} if no other sub-TPO of $\ctpo$ with the same entry $e_{\mathrm{in}}$ has a strictly larger event set.
\end{definition}
Locally maximal sub-TPOs are not only well-defined but also pairwise disjoint.

\begin{myproposition}[Disjointness of Locally Maximal Sub-TPOs]
    \label{prop:unique}
    For any two distinct locally maximal sub-TPOs $\stpo$ and $\stpo'$, $\events_s \cap \events_{s'} = \emptyset$.
\end{myproposition}
\begin{proof}
Suppose $\events_s \cap \events_{s'} \neq \emptyset$.
If 
$\events_s \subseteq \events_{s'}$ or $\events_{s'} \subseteq \events_{s}$, local maximality is violated at the shared or inner entry.
If they partially overlap, any shared event forces edges or precedence relations across a sub-TPO boundary, violating GTSP-TWPR or sub-TPO isolation, or else the two sub-TPOs concatenate into a strictly larger valid sub-TPO, violating local maximality. Every case yields a contradiction.
\end{proof}

We now present Algorithm~\ref{alg:detection} to compute the set of locally maximal sub-TPOs. We define the set of \emph{articulation events} $\mathcal{A}(\gtsptwpr) = \{e_i \in \events \mid \exists v \in V_{l_i} $\text{ s.t. }$ v \text{ is a strong articulation point of } \gtsptwpr\}$ as events whose removal from $\gtsptwpr$ increases the number of strongly connected components, and the \emph{initial events} $\mathcal{I}(\gtsptwpr) := \{e_i \in \events \mid \exists\, v \in V_{l_i},\, (v_0, v) \in E\}$ as the set of events that are reachable directly from the depot $v_0$. 
In the construction of $\gtsptwpr$, each event node carries a self-loop to enable correct articulation point identification. A valid entry event $e_{\mathrm{in}}$ must be either an initial event or an articulation event: if it were neither, its internal events would remain reachable from the rest of $\gtsptwpr$ after its removal, violating GTSP-TWPR isolation. This restricts the candidate set to $\mathcal{C} = \mathcal{A}(\gtsptwpr) \cup \mathcal{I}(\gtsptwpr) \subseteq \events$. 

Algorithm~\ref{alg:detection} iterates over $\mathcal{C}$ in topological order and, for each unclassified candidate entry $e$, partitions the events reachable from $e$ under $\prec$ into descendant sets $\{B_i\}$ by $e$'s direct successors. It then applies the pruning procedure \textsc{IsSubTPO} (described below) to the full union $B = \cup_i B_i$, which returns the largest valid subset of $B \cup \{e\}$ satisfying Def.~\ref{def:stpo}, if one exists. Once a sub-TPO is found, all its events are marked and excluded from 
any subsequent sub-TPO.


\setlength{\textfloatsep}{6pt} 
\begin{algorithm}[t]
\small
\caption{Maximal Sub-TPO Decomposition}
\label{alg:detection}
\KwIn{$\ctpo$, $\gtsptwpr = (V, E)$}
\KwOut{$\mathcal{M}$: set of maximal sub-TPOs of $\ctpo$}
$\mathcal{M} \leftarrow \emptyset$\;
$\mathrm{classified} \leftarrow \emptyset$\;
$\mathcal{C} \leftarrow$ topological order of $\mathcal{A}(\gtsptwpr) \cup \mathcal{I}(\gtsptwpr)$ induced by $\prec$\;
\For{$e \in \mathcal{C}$}{
    \If{$e \notin \mathrm{classified}$}{
        $\{B_1, \ldots, B_m\} \leftarrow \textsc{Descendants}(e, \events, \prec)$\;
        $B \leftarrow \bigcup_{i=1}^{m} B_i$\;
        \If{$\textsc{IsSubTPO}(B \cup \{e\}, e, \gtsptwpr, \ctpo)$}{
            $\mathcal{M} \leftarrow \mathcal{M} \cup \{\stpo \text{ induced by } B \cup \{e\}\}$\;
            $\mathrm{classified} \leftarrow \mathrm{classified} \cup B \cup \{e\}$\;
        }
    }
}
\Return{$\mathcal{M}$}\;
\end{algorithm}


\textsc{IsSubTPO} takes the candidate set $B \cup {e}$ and iteratively removes any event that violates one of the isolation conditions of Def.~\ref{def:stpo} relative to the current set. This process repeats until the candidate set contains no violations. The procedure returns the resulting set as a valid sub-TPO if it contains at least two events and has exactly one exit (an event with an edge to an event outside the set); otherwise, it rejects the set.

\begin{mytheorem}[Completeness]
\label{thm:correctness}
Algorithm~\ref{alg:detection} is complete,~i.e.,~it returns exactly the set of all locally maximal sub-TPOs of~$\ctpo$.
\end{mytheorem}
\begin{proof}[Proof]
Any locally maximal sub-TPO's entry $e_{\mathrm{in}}$ is either an initial event or, by GTSP-TWPR isolation, an articulation event, so $e_{\mathrm{in}} \in \mathcal{C}$. Topological processing of $\mathcal{C}$ guarantees $e_{\mathrm{in}}$ is considered before any of its descendants can be classified by another sub-TPO. Since every event of a valid sub-TPO satisfies Def.~\ref{def:stpo}, none are pruned by \textsc{IsSubTPO}, which therefore returns the largest valid subset, ensuring local maximality. Proposition~\ref{prop:unique} ensures no two locally maximal sub-TPOs compete for the same events.
\end{proof}

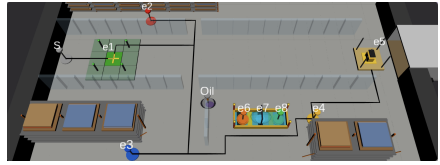
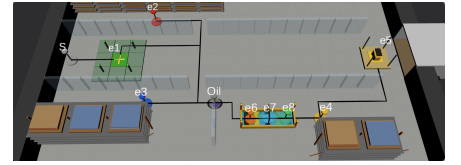
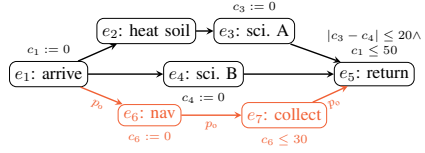
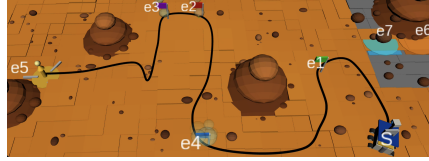
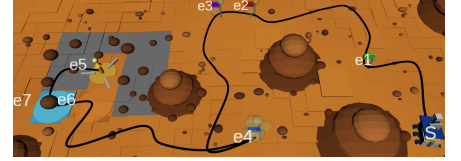
\begin{figure*}[t]
    \centering
    
    \begin{subfigure}[b]{0.32\linewidth}
        \centering
        \resizebox{\linewidth}{!}{
        \begin{tikzpicture}[
            node distance=0.5cm and 2.5cm,
            every node/.style={draw, rounded corners, fill=white!20, minimum width=2em, inner sep=3pt},
            >=stealth
        ]
        \node (e1) {$e_1$: charge};
        \node (e3) [right=2 of e1] {$e_3$: shelf B};
        \node (e2) [above=0.4 of e3] {$e_2$: shelf A};
        \node (e4) [below=0.4 of e3] {$e_4$: shelf C};
        \node (e5) [right=2 of e3] {$e_5$: dock};

        \node (e7) [below=0.2 of e4, draw=RedOrange, text=RedOrange] {$e_7$: clean 2};
        \node (e6)  [left=0.75 of e7, draw=RedOrange, text=RedOrange] {$e_6$: clean 1};
        \node (e8)  [right=0.75cm of e7, draw=RedOrange, text=RedOrange] {$e_8$: clean 3};

        \draw[->, thick] (e1) -- (e2);
        \draw[->, thick] (e1) -- (e3);
        \draw[->, thick] (e1) -- (e4);
        \draw[->, thick] (e2) -- (e5);
        \draw[->, thick] (e3) -- (e5);
        \draw[->, thick] (e4) -- (e5);

        \draw[->, thick, RedOrange] (e1) -- (e6)
            node[midway, below, sloped, label_node, font=\scriptsize]{\textcolor{RedOrange}{$p_h$}};
        \draw[->, thick, RedOrange] (e6) -- (e7)
            node[midway, below, label_node, font=\scriptsize]{\textcolor{RedOrange}{$p_h$}};
        \draw[->, thick, RedOrange] (e7) -- (e8)
            node[midway, below, label_node, font=\scriptsize]{\textcolor{RedOrange}{$p_h$}};
        \draw[->, thick, RedOrange] (e8) -- (e5)
            node[midway, below, sloped, label_node, font=\scriptsize]{\textcolor{RedOrange}{$p_h$}};

        \node[label_node, above=0.005cm of e1, font=\scriptsize] {$c_1 := 0$};
        \node[label_node, above=0.005cm of e2, font=\scriptsize] {$5 \leq c_1 \leq 15$};
        \node[label_node, above=0.005cm of e3, font=\scriptsize, align=center] {$c_1 \geq 5$};
        \node[label_node, above=0.005cm of e4, font=\scriptsize, align=center] {$c_1 \geq 5$};
        \end{tikzpicture}
        }
        \caption{Task cTPO $\ctpo^\text{insp}$. Black edges and guards apply unconditionally. \textcolor{RedOrange}{Orange} denotes the conditional events.}
        \label{fig:warehouse_ctpo}
    \end{subfigure}
    \hfill
    \begin{subfigure}{0.32\linewidth}
        \centering
        \includegraphics[width=\linewidth]{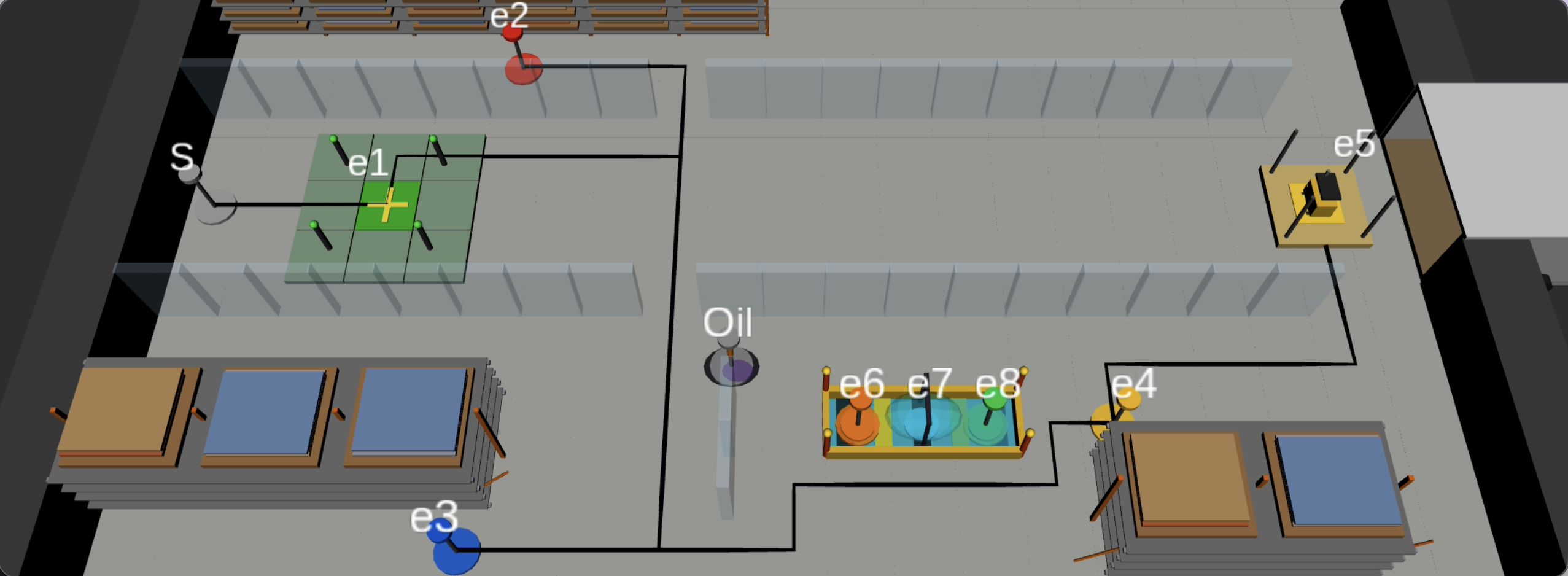}
        \caption{The optimal trajectory avoids the spill zone and completes only the mandatory, unconditional events of $\ctpo^\text{insp}$.}
        \label{fig:warehouse_shortcut}
    \end{subfigure}%
    \hfill
    \begin{subfigure}{0.32\linewidth}
        \centering
        \includegraphics[width=\linewidth]{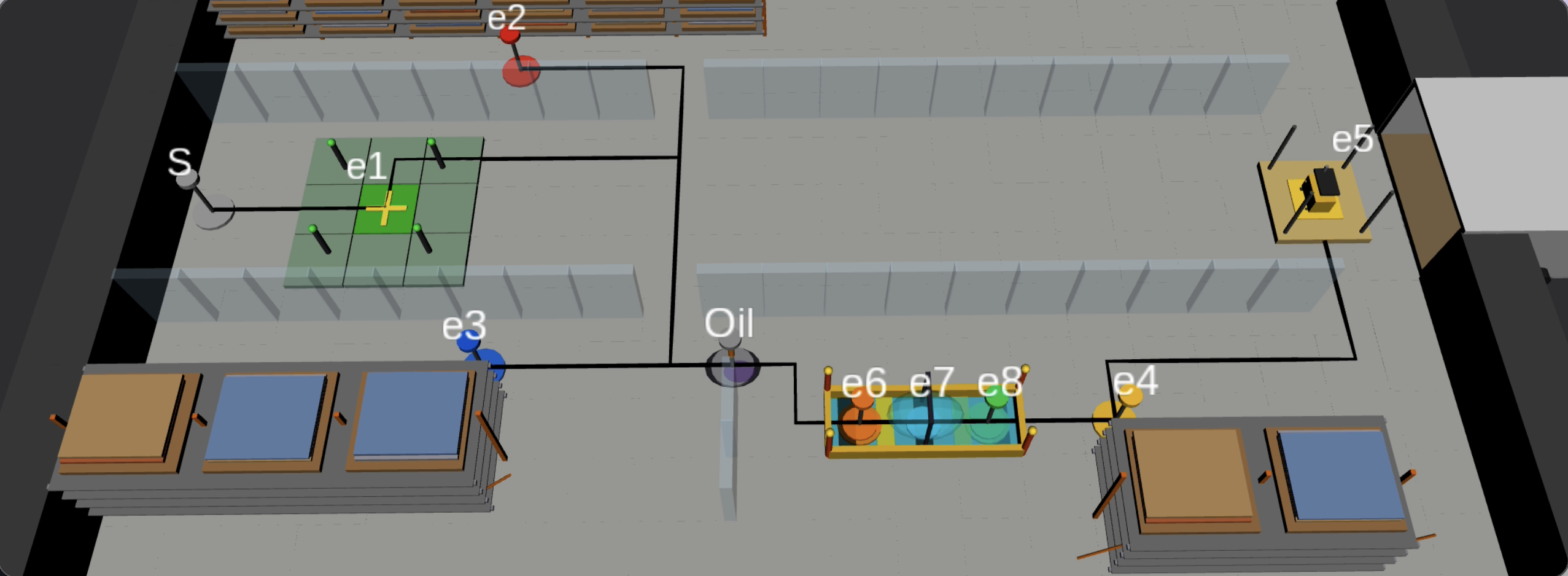}
        \caption{The optimal traj. passes through the spill zone, activating $p_h = \top$ and extending the plan with the cleaning sequence $e_6, e_7, e_8$.}
        \label{fig:warehouse_long}
    \end{subfigure}%
    %

    \begin{subfigure}{0.32\linewidth}
        \centering
        \resizebox{\linewidth}{!}{
        \begin{tikzpicture}[
            node distance=0.5cm and 3.5cm,
            every node/.style={draw, rounded corners, fill=white!20, minimum width=2em, inner sep=3pt},
            >=stealth
        ]
        \node (e1) {$e_1$: arrive};
        \node (e4) [right=1.5 of e1] {$e_4$: sci. B};
        \node (e5) [right=1.8 of e4] {$e_5$: return};
        \node (e2) [above=0.6 of $(e1)!0.3!(e5)$] {$e_2$: heat soil};
        \node (e3) [right=0.35cm of e2] {$e_3$: sci. A};
        \node (e6) [below=0.6 of $(e1)!0.31!(e5)$, draw=RedOrange, text=RedOrange] {$e_6$: nav};
        \node (e7) [right=1.2cm of e6, draw=RedOrange, text=RedOrange] {$e_7$: collect};

        \draw[->, thick] (e1) -- (e2);
        \draw[->, thick] (e2) -- (e3);
        \draw[->, thick] (e3) -- (e5);
        \draw[->, thick] (e1) -- (e4);
        \draw[->, thick] (e4) -- (e5);
        \draw[->, thick, RedOrange] (e1) -- (e6) node[midway, below, sloped, label_node, font=\scriptsize] {\textcolor{RedOrange}{$p_{\text{o}}$}};
        \draw[->, thick, RedOrange] (e6) -- (e7) node[midway, below, sloped, label_node, font=\scriptsize] {\textcolor{RedOrange}{$p_{\text{o}}$}};
        \draw[->, thick, RedOrange] (e7) -- (e5) node[midway, below, sloped, label_node, font=\scriptsize] {\textcolor{RedOrange}{$p_{\text{o}}$}};
        \node[label_node, above=0.01cm of e1, font=\scriptsize] {$c_1 := 0$};
        \node[label_node, above=0.01cm of e3, font=\scriptsize] {$c_3 := 0$};
        \node[label_node, below=0.01cm of e4, font=\scriptsize] {$c_4 := 0$};
        \node[label_node, below=0.01cm of e6, font=\scriptsize] {\textcolor{RedOrange}{$c_6 := 0$}};
        \node[label_node, below=0.01cm of e7, font=\scriptsize] {\textcolor{RedOrange}{$c_6 \leq 30$}};
        \node[
        label_node,
        above=0.01cm of e5,
        font=\scriptsize,
        align=center
    ]{
    $|c_3 - c_4| \leq 20 \wedge$ \\
        $c_1 \leq 50$
    };
        \end{tikzpicture}
        }
        \caption{Task cTPO $\ctpo^\text{rover}$. 
        Black edges and guards apply unconditionally. \textcolor{RedOrange}{Orange} denotes the conditional requirements.
        }
        \label{fig:rover_ctpo}
    \end{subfigure}
    \hfill
    \begin{subfigure}{0.32\linewidth}
        \centering
        \includegraphics[width=\linewidth]{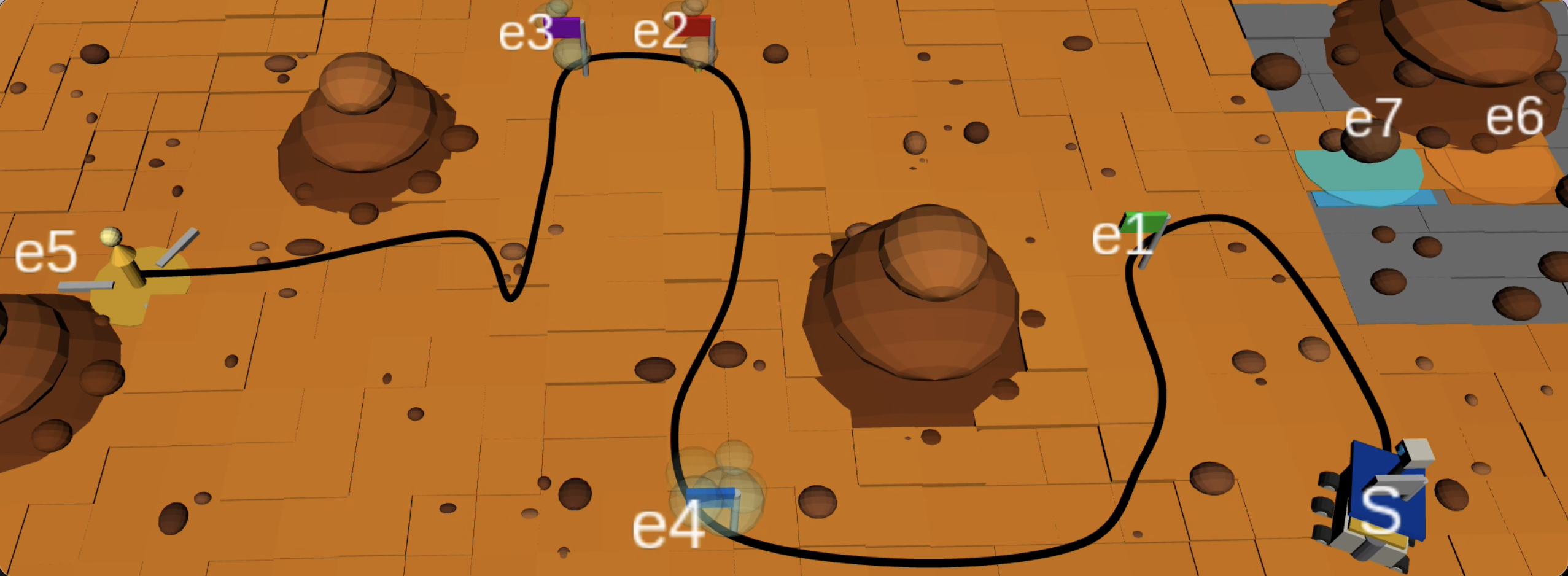}
        \caption{The outcrop is away from the rover. The optimal trajectory avoids it, yielding $p_\text{o} = \bot$ and a plan over five events only.}
        \label{fig:mars_no_outcrop}
    \end{subfigure}%
    \hfill
    \begin{subfigure}{0.32\linewidth}
        \centering
        \includegraphics[width=\linewidth]{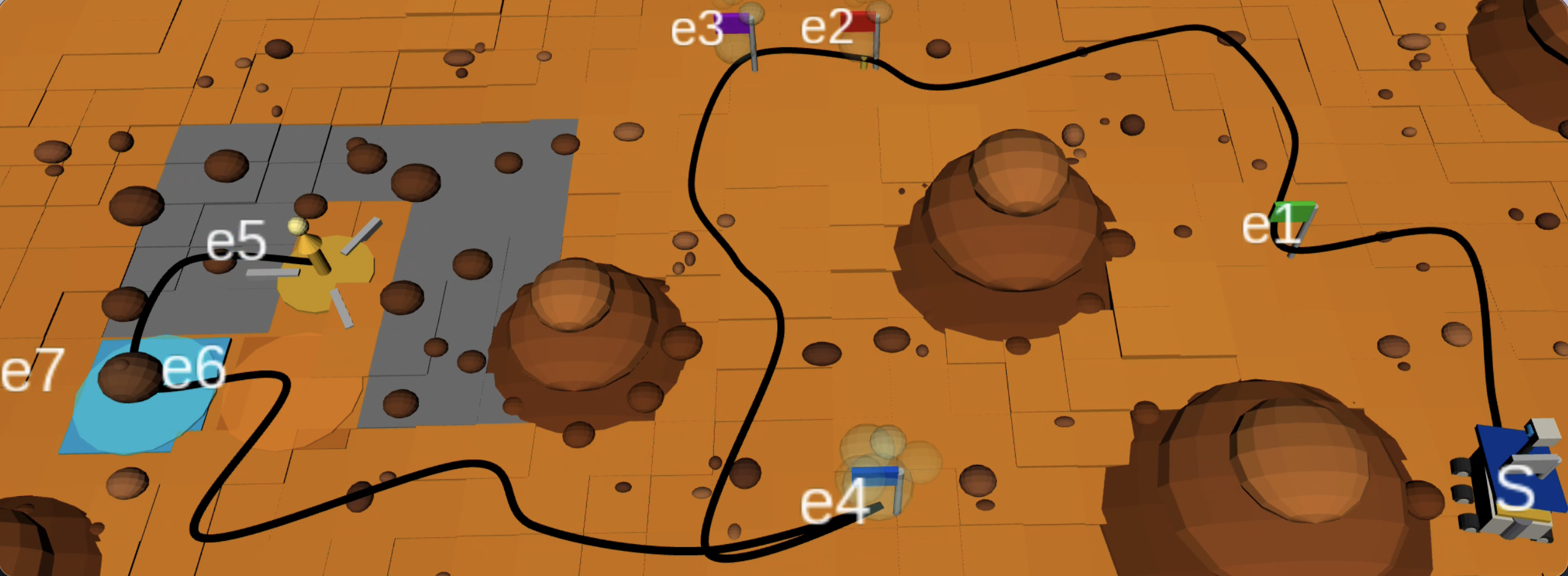}
        \caption{The outcrop surrounds the rover. The optimal trajectory passes through it, activating $p_\text{o} = \top$ and extending the plan with $e_6$, $e_7$.}
        \label{fig:mars_outcrop}
    \end{subfigure}%

    \vspace{-1mm}
    \caption{
    \small
    Illustrative examples.
    Task cTPOs and plans 
    for the Warehouse (a)-(c) and Mars Rover (d)-(f) examples (video: \url{https://youtu.be/9tdEdERKRE4}). 
    }
    \label{fig:warehouse_all}
    \vspace{-3mm}
\end{figure*}

Once all maximal sub-TPOs are identified, we solve each ${\stpo}_i \in \mathcal{M}$ independently, obtaining a local timed trace $\tau_i^*$ with duration $d_i^*$. We then solve a reduced \emph{high-level} TPO $\tpo_H$, which replaces each sub-TPO with its entry event $e_{\mathrm{in}}^i$ and adjusts the travel cost and timing bounds around $e_{\mathrm{in}}^i$ by $d_i^*$ to account for the time spent inside it. The full trace is recovered from the resulting high-level trace $\tau_H^*$ by substituting each local $\tau_i^*$ back in at the time $e_{\mathrm{in}}^i$ is reached.

\begin{mytheorem}[Optimality]
\label{thm:optimality}
The timed trace $\tau_H^*$ returned by the decomposition procedure above is optimal for Problem~\ref{prob:synthesis}.
\end{mytheorem}
\begin{proof}
    By Conditions~\ref{cond: G isolocation}--\ref{cond: activation consistency} of Def.~\ref{def:stpo}, each ${\stpo}_i$ is isolated from $\ctpo$ beyond $e_{\mathrm{in}}^i, e_{\mathrm{out}}^i$, so the cost of traversing $\events_i$ is independent of the rest of the plan and $d_i^*$ equals what any globally optimal solution would incur on $\events_i$. Substituting $d_i^*$ into the high-level problem is thus exact, so solving $\tpo_H$ optimally yields $\tau_H^*$ with makespan equal to the monolithic optimum.
\end{proof}

\section{Experiments}\label{sec:experiments}

Here, we provide empirical evaluations of our cTPO framework in two stages: illustrative examples and benchmarks.
The examples are designed to illustrate interpretability of cTPOs, correctness of synthesized plans, and how a change in the environment changes optimal plans depending on proposition evaluations. 
The benchmarks compare the decomposition approach against the monolithic MILP on randomly generated problems.

Our tool is implemented in Python, uses Gurobi~\cite{gurobi} as the MILP solver, and integrates with ROS for robot deployment. All benchmarks were conducted on a MacBook Air M3 with 16 GB of RAM and robot experiments performed in ROS with RViz visualization. 




\subsection{Illustrative Examples}

\subsubsection{Warehouse Inspection}
We consider a robot inspecting three warehouse stations, Shelf~A ($e_2$), Shelf~B ($e_3$), and Shelf~C ($e_4$), and returning to its loading dock ($e_5$) after a mandatory charging event ($e_1$). The aisle layout contains an oil spill zone. If the robot traverses the spill zone, it must complete a cleaning sequence ($e_6, e_7, e_8$) before proceeding to the loading dock. 
The cTPO $\ctpo^\text{insp}$  is shown in Fig.~\ref{fig:warehouse_ctpo}, and the 
environments are shown in Fig.~\ref{fig:warehouse_shortcut}-\ref{fig:warehouse_long}, which
differ only in the position of Shelf~B relative to the spill zone.  The size of the robot's DTS and induced GTSP-TWPR graph, along with the MILP planning times are shown in Table~\ref{tab:case_study_sizes}.

Notice that  the environment topology influences the optimal trajectories and how they satisfy $\ctpo^\text{insp}$. Specifically, in Fig.~\ref{fig:warehouse_shortcut}, the location of Shelf B ($e_3$) enables the robot to optimally avoid the spill zone and satisfy $\ctpo^\text{insp}$ without triggering $p_h$, whereas in Fig.~\ref{fig:warehouse_long}, the optimal behavior is to pass through the spill zone (triggering $p_h$) and go through the three cleaning areas.

\begin{table}[t]
    \centering
    \caption{
    \small
    Robot's DTS $\dts$ and the induced GTSP-TWPR graph $\gtsptwpr$ sizes and MILP solve time (monolithic) 
    for the illustrative examples. Synthesis times are averaged over 3 trials. 
    }
    \label{tab:case_study_sizes}
    \resizebox{\columnwidth}{!}{%
    \begin{tabular}{l cc cc cc}
        \hline
        & \multicolumn{2}{c}{\underline{\qquad$\dts$ \qquad}} & \multicolumn{2}{c}{\underline{\qquad $\gtsptwpr$ \qquad}} & 
        Synth. \\
        & $|X|$ & $|\delta|$ & $|V|$ & $|E|$ &  Time (s)\\
        \hline
        Warehouse Fig.~\ref{fig:warehouse_shortcut} & 108 & 728 & 10 & 31 & 0.010 \\
        Warehouse Fig.~\ref{fig:warehouse_long} & 108 & 728 & 10 & 31 & 0.008\\
        Rover Fig.~\ref{fig:mars_no_outcrop} & 70 & 280 & 15 & 165 & 77.5 \\
        Rover Fig.~\ref{fig:mars_outcrop} & 70 & 280 & 14 & 139 & 575.6 \\
        \hline
    \end{tabular}}
\end{table}

\subsubsection{Mars Rover Atmospheric Science Mission}
We consider a Mars rover mission in which the rover must, after arriving at the landing site ($e_1$), heat and analyze a soil sample ($e_2 \prec e_3$) and, independently, take a measurement with instrument B ($e_4$), with both branches completing within 20 time units of each other. Then, it must return to its lander ($e_5$) within 50 time units of arrival. 
If the rover ever passes through an outcrop, before returning, it must ensure to have navigated to $e_6$ and collected a sample from $e_7$ within 30 time units. The cTPO representation of the task $\ctpo^\text{rover}$ is shown in Fig.~\ref{fig:rover_ctpo}.

For plan synthesis, we consider the environments in Figs.~\ref{fig:mars_no_outcrop}-\ref{fig:mars_outcrop}, differing in the placement and number of the outcrop regions relative to the lander. 
Table~\ref{tab:case_study_sizes} summarizes the problem size and plan synthesis times.
In Fig.~\ref{fig:mars_no_outcrop}, the optimal trajectory avoids the outcrop ($p_\text{o} = \bot$) and visits unconditional events only.
In Fig.~\ref{fig:mars_outcrop}, it is optimal for the rover to pass through the outcrop, setting $p_\text{o} = \top$ and executing $e_6$ and $e_7$ rather than avoiding it. In fact, it cleverly executes $e_6$ and $e_7$ before triggering $p_\text{o}$, anticipating its eventual activation along the trajectory.


\subsection{Benchmarks}

    

\begin{figure}[t]
    \centering
    \begin{subfigure}{0.95\linewidth}
        \centering
        \includegraphics[width=\linewidth]{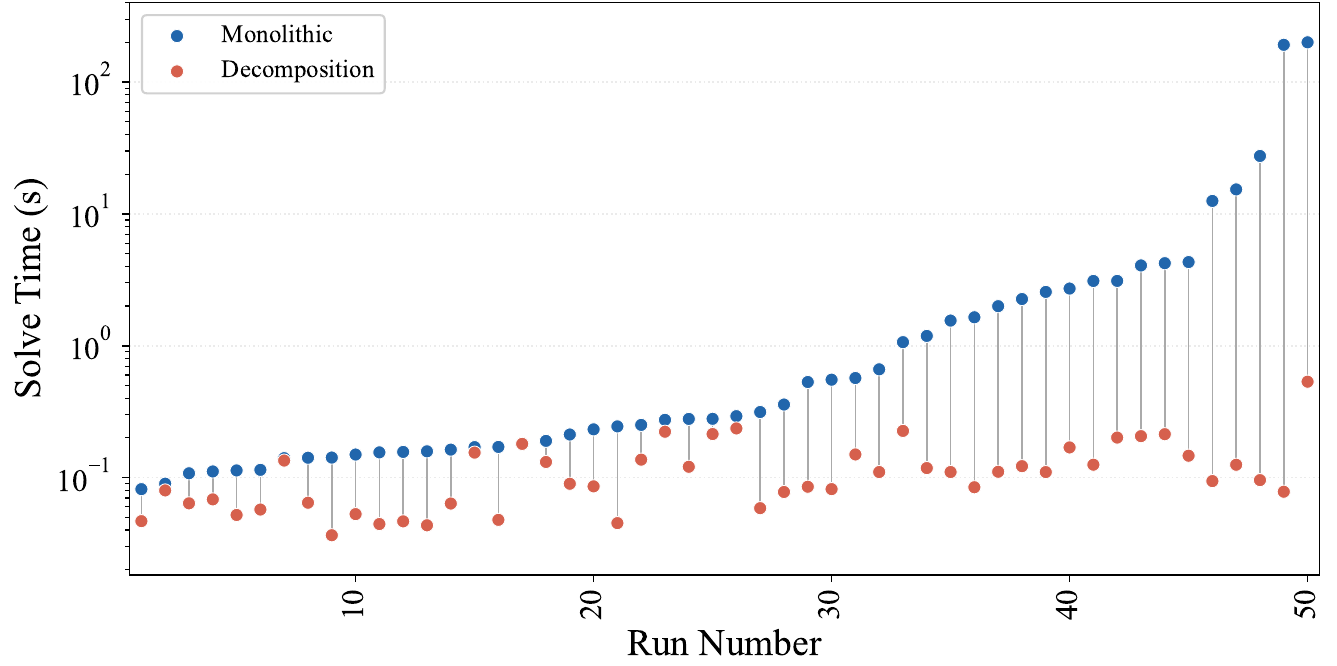}
        \vspace{-6mm}
        \caption{TPOs solve times across 50 runs.}
        \label{fig:plot_speedup_tpo}
    \end{subfigure}%
    
    \begin{subfigure}{0.95\linewidth}
        \centering
        \includegraphics[width=\linewidth]{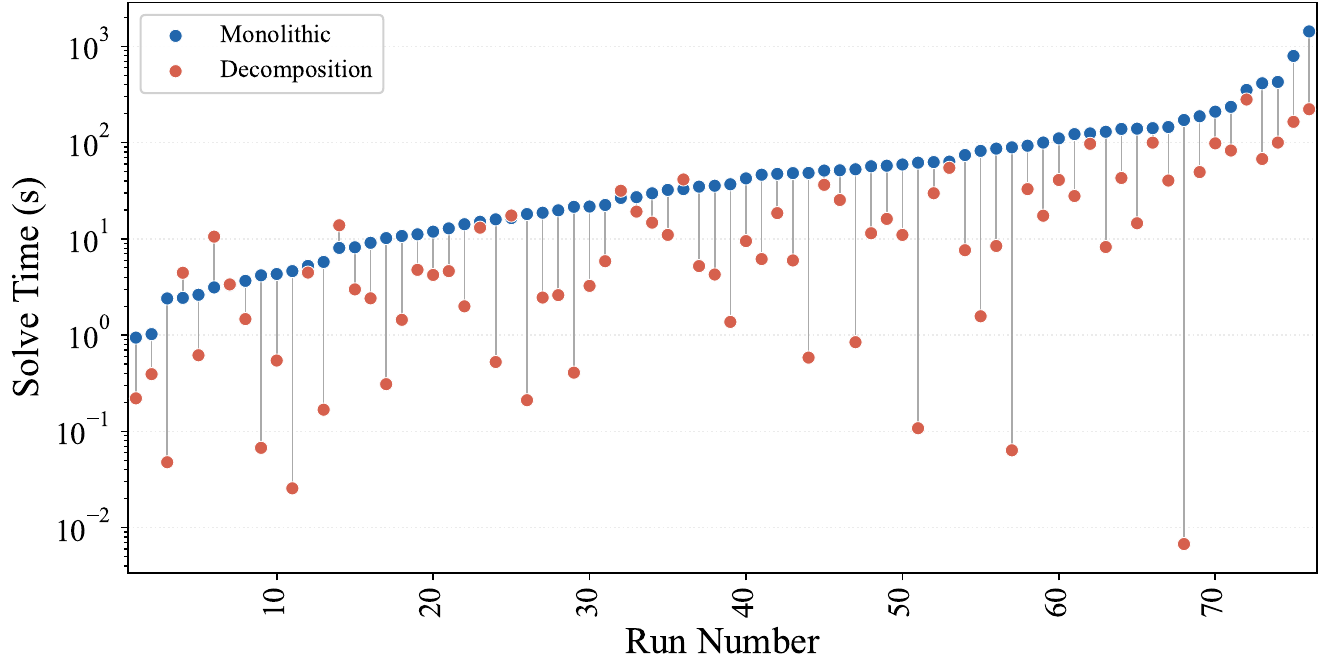}
        \vspace{-6mm}
        \caption{cTPO solve times across 76 runs.}
        \label{fig:plot_speedup_ctpo}
    \end{subfigure}
    \vspace{-1mm}
    \caption{
    \small
    Benchmark results of monolithic  MILP vs decomposition-based MILP.
    Each data point is averaged over 3 trials. 
    }
    \label{fig:decomp_results}
    \vspace{-1mm}
\end{figure}

\subsubsection{Sub-TPO Decomposition on Random TPOs}

This benchmark isolates the effect of decomposition from conditional structure by considering TPOs with no proposition nodes or conditional sub-TPOs.
%
We randomly generated TPOs, each with 50-60 events, until 50 feasible TPOs were obtained. 
We then solved each feasible instance using both the monolithic MILP and decomposition-based approaches. The results (averaged over 3 trials) are reported in Fig.~\ref{fig:plot_speedup_tpo}, with instances sorted by increasing monolithic solve time along the x-axis.

Note that as instance difficulty increases, the monolithic solve time grows by several orders of magnitude, while the decomposition-based solve time remains stable at roughly below $1$s. Consequently, the speedup increases with problem difficulty, exceeding $3$ orders of magnitude for the most complex instances. The sharpest increases in monolithic solve time correspond to instances with sparser timing constraints, which reduce MILP pruning and enlarge the feasible region, consistent with the complexity analysis in Sec.~\ref{sec:milp}.

\subsubsection{Sub-TPO Decomposition on Random cTPOs}
The second benchmark extends the evaluation to cTPOs with conditional sub-TPOs, totaling 13-22 events and 
4-13 atomic propositions.
Fig.~\ref{fig:plot_speedup_ctpo} shows the results across 76 feasible instances sorted by increasing monolithic solve time. 
Note that decomposition remains advantageous in over $90\%$ of instances, with speedups reaching $4$ orders of magnitude on the hardest instance. 
Unlike the TPO case, decomposition times also grow with problem difficulty, as the high-level MILP must jointly resolve proposition values and tour structure. The variability in speedup reflects the difficulty of this high-level MILP, which depends on the randomly generated proposition formulas and activation structure.

We finally note that, in all TPO and cTPO benchmarks, the decomposition-based method matched the monolithic solution cost (plan makespan), validating Theorem~\ref{thm:optimality}.


\section{Conclusion}
In this work, we introduced Conditional TPOs to address the expressiveness limitations of existing TPOs in robotic task planning. By incorporating generalized clock-difference guards and environment-dependent activation conditions, we enabled the specification of complex timing constraints between unordered events alongside conditionally activated task structures. To efficiently solve the resulting plan synthesis problem, we extended existing MILP formulations and introduced a novel sub-TPO decomposition algorithm. Our benchmarks and case studies demonstrate that this decomposition preserves plan optimality, yields computational speedups of up to 4 orders of magnitude for complex instances.
Future work includes extension of cTPOs multi-agent planning.

\bibliographystyle{IEEEtran}
\bibliography{references}

\end{document}